\documentclass{article}
\newif\ifsup\suptrue
\usepackage{microtype}
\usepackage{booktabs} 

\usepackage{hyperref}

\usepackage{amsmath,amssymb,amsthm,dsfont,graphicx,xcolor,bm}
\usepackage[capitalize,noabbrev]{cleveref}

\usepackage[lofdepth,lotdepth]{subfig}

 \usepackage{algorithm} 
\usepackage{algorithmic}

\newtheorem{theorem}{Theorem}

\newtheorem{remark}{Remark}

\DeclareMathOperator*{\argmax}{arg\,max\,}
\DeclareMathOperator*{\argmin}{arg\,min\,}
\mathchardef\mhyphen="2D

\newcommand{\cA}{\mathcal{A}}

\newcommand{\cB}{\mathcal{B}}

\newcommand{\cD}{\mathcal{D}}

\newcommand{\cN}{\mathcal{N}}
\newcommand{\R}{\mathds{R}}

\newcommand{\thetauotf}{\hat{\theta}^{\textsc{wf}}}
\newcommand{\thetacotf}{\hat{\theta}^{\textsc{w}}}

\newcommand{\LinUCB}{{\tt LinUCB}}

\newcommand{\uOTF}{{\textsc{wf-otf}}}

\newcommand{\OTFLinUCB}{{\tt OTFLinUCB}}
\newcommand{\OTFLinTS}{{\tt OTFLinTS}}

\newcommand{\inner}[1]{\left\langle #1 \right\rangle}
\newcommand{\shortinner}[1]{\langle #1 \rangle}
\renewcommand{\R}{\mathds{R}}
\newcommand{\N}{\mathds{N}}
\newcommand{\norm}[1]{\left\Vert #1 \right\Vert}
\newcommand{\snorm}[1]{\Vert #1 \Vert}

\newcommand{\E}{\mathds E}
\newcommand{\PP}{\mathds P}

\usepackage[accepted]{icml2020_tweak}

\icmltitlerunning{Linear Bandits with Stochastic Delayed Feedback}

\begin{document}

\twocolumn[
\icmltitle{Linear Bandits with Stochastic Delayed Feedback}




\begin{icmlauthorlist}
\icmlauthor{Claire Vernade}{dm}
\icmlauthor{Alexandra Carpentier}{unimag}
\icmlauthor{Tor Lattimore}{dm}
\icmlauthor{Giovanni Zappella}{ama}
\icmlauthor{Beyza Ermis}{ama}
\icmlauthor{Michael Brueckner}{ama}
\end{icmlauthorlist}

\icmlaffiliation{dm}{DeepMind, London, UK}
\icmlaffiliation{unimag}{Otto-Von-Guericke Universität, Magdeburg, Germany}
\icmlaffiliation{ama}{Amazon, Berlin, Germany}

\icmlcorrespondingauthor{Claire Vernade}{vernade@google.com}

\icmlkeywords{Stochastic Bandits, Delayed Feedback, Regret Minimization, Linear Bandits}

\vskip 0.3in
]



\printAffiliationsAndNotice{}  


\begin{abstract}
Stochastic linear bandits are a natural and well-studied model for structured exploration/exploitation problems and are widely used in applications such as online marketing and recommendation.
One of the main challenges faced by practitioners hoping to apply existing algorithms is that usually the feedback is randomly delayed and delays are only partially observable.
For example, while a purchase is usually observable some time after the display, the decision of not buying is never explicitly sent to the system.
In other words, the learner only observes delayed positive events.
We formalize this problem as a novel stochastic delayed linear bandit and propose $\OTFLinUCB$ and $\OTFLinTS$, two computationally efficient algorithms able to integrate new information as it
becomes available and to deal with the permanently censored feedback. We prove optimal $\tilde O(\smash{d\sqrt{T}})$ bounds on the regret of the first algorithm and study the dependency on delay-dependent parameters.
Our model, assumptions and results are validated by experiments on simulated and real data.
\end{abstract}

\section{Introduction}
\label{sec:introduction}

Content optimization for websites and online advertising are among the main industrial applications of bandit algorithms \cite{chapelle2011empirical,chapelle2014modeling}.
The dedicated services sequentially choose an option among several possibilities and display it on a web page to a particular customer.
In most real world architectures, for each recommendation request,
the features of the products are joined and hashed with those of the current user and provide a (finite) action set included in  $\R^d$.
For that purpose, linear bandits \cite{chu2011contextual,abbasi2011improved} are among the most adopted as they allow to take into account the structure of the space where the action vectors lie.

A key aspect of these interactions through displays on webpages is the time needed by a customer to make a decision and provide feedback to the learning
algorithm, also known as the \emph{conversion indicator} \cite{chapelle2014modeling,DiemertMeynet2017}.
For example, a mid-size e-commerce website can serve hundreds of recommendations per second, but customers need minutes, or even hours, to make a purchase.
In \cite{chapelle2014modeling}, the authors ran multiple tests on proprietary industrial datasets, providing a good example of how delays affect the performance of click-through rate estimation.
They extract 30 days of display advertising data and find that delays are on average of the order of several hours and up to several days.

On the other extreme of the time scale, some companies optimize long-term metrics for customer engagement (e.g., accounting for returned products in the sales results)
which by definition can be computed only several weeks after the bandit has played the action.
Moreover, after a piece of content is displayed on a page, the user may or may not decide to react (e.g., click, buy a product). In the negative case, no signal is sent to the system and the
learner cannot distinguish between actions for which the user did not click and those where they did, but the conversion is delayed. Note that in the use cases we consider, when a feedback is received, the learner is able to attribute it to the past action that triggered it.

Two major requirements for bandit algorithms in order to be ready to run in a real online service are the ability to leverage contextual information and handle
delayed feedback. Many approaches are available to deal with contextual information
\cite{abbasi2011improved,agarwal2014taming,auer2002nonstochastic,neu2015explore,beygelzimer2011contextual,chu2011contextual,zhou2015survey}. Delays
have been identified as a major problem in online applications \cite{chapelle2014modeling}.
We give an overview of the existing literature in Section~\ref{sec:related}.
However, to the best of our knowledge, no algorithm was able to address this problem given the requirements defined above.

\paragraph{Contributions}
Our main contribution is a novel bandit algorithm called \texttt{On-The-Fly-LinUCB} ($\OTFLinUCB$).
The algorithm is based on $\LinUCB$ \cite{abbasi2011improved}, but with confidence intervals and least-squares estimators that are adapted to account for the delayed and censored rewards (\cref{sec:concentration}).
The algorithm is complemented by a regret analysis, including lower bounds (\cref{sec:analysis}).
We then provide a variant inspired by Thompson sampling and Follow the Perturbed Leader (\cref{sec:ts}) and
evaluate the empirical performance of all algorithms in \cref{sec:experiments}.

\section{Learning Setting under Delayed Feedback}
\label{sec:setting}

\paragraph{Notation}
All vectors are in $\R^d$ where $2\leq d<\infty$ is fixed.
For any symmetric positive definite matrix $M$ and vector $x\in \R^d$, $\|x\|_M = \sqrt{x^T M x}$ and $\norm{x}_2=\|x\|_I$ is the usual $L_2$-norm of $x$, where $I$ denotes the identity matrix in dimension $d$.

\paragraph{Learning setting}
Our setup involves a learner interacting with an environment over $T$ rounds.
The environment depends on an unknown parameter $\theta \in \R^d$ with $\norm{\theta}_2 \leq 1$ and a delay distribution $\cD$ supported on the natural numbers.
Note, in contrast to \citet{vernade2017stochastic}, we do not assume the learner knows $\cD$.
Then, in each round $t$,
\begin{enumerate}
\item The learner receives from the environment a finite set of actions $\cA_t \subset \R^d$ with $|\cA_t| = K_t < \infty$ and $a^\top \theta \in [0,1]$ and $\norm{a}_2 \leq 1$ for all $a \in \cA_t$.
\item The learner selects an action $A_t$ from $\cA_t$ based on information observed so far.
\item The environment samples a reward $X_t \in \{0,1\}$ and a delay $D_t \in \N$ that are \textit{partially} revealed to the learner and where:
\subitem 3.a.\, $X_t \sim \cB(A_t^\top \theta)$.
\subitem 3.b.\, $D_t$ is sampled independently of $A_t$ from $\cD$.
\item Certain rewards resulting from previous actions are revealed to the learner. For $s \leq t$, let $C_{s,t} = \mathds{1}\{D_s \leq t - s\}$, which is
called the censoring variable and indicates whether or not the reward resulting from the decision in round $s$ is revealed by round $t$. Then let
$Y_{s,t} = C_{s,t} X_s$. The learner observes the collection $\{Y_{s,t} : s \leq t\}$ at the end of round $t$. If $X_t =1$, we say that the action $A_t$ converts.
\end{enumerate}
The delays in combination with the Bernoulli noise model and censored observations introduces an interesting structure.
When $X_s = 0$, then $Y_{s,t} = 0$ for all $t \geq s$, but $Y_{s,t} = 0$ is also possible when $X_s = 1$, but the reward from round $s$ has been delayed sufficiently. On the other hand, if $Y_{s,t} = 1$, the learner can immediately deduce that $X_s = 1$.

The goal of the learner is to sequentially minimize the cumulative regret after $T$ rounds, which is 
\begin{equation}
  \label{eq:regret}
R(T, \theta) = \sum_{t=1}^T \inner{\theta, A^*_t}-\inner{\theta, A_t},
\end{equation}
where $A_t^* = \argmax_{a\in \mathcal A_t} \inner{\theta, a}$ is the action that maximises the expected reward in round $t$.

\begin{remark}
The assumption that $a^\top \theta \in [0,1]$ for all $a \in \cA_t$ ensures the reward is well defined. A natural alternative is to replace the linear model with a generalized linear model. Our algorithms and analysis generalize to this setting in the natural way using the techniques of \cite{filippi2010parametric, jun2017scalable}. For simplicity, however, we restrict our attention to the linear model.
The assumption that $\norm{\theta}_2 \leq 1$ and $\norm{a}_2 \leq 1$ for all actions are quite standard and the dependence of our results on alternative bounds is relatively mild.
\end{remark}

\section{Concentration for Least Squares Estimators with Delays}
\label{sec:concentration}
This section is dedicated to disentangling the delays and the reward estimation.
The combination of an unknown delay distribution and censored binary rewards makes it hopeless to store all past
actions and wait for every conversion. For this reason our algorithm depends on a parameter $m$. If a reward has not
converted within $m$ rounds, the algorithm assumes it will never convert and ignores any subsequent signals related
to this decision. There is also a practical advantage, which is that the learner does not need to store individual
actions that occurred more than $m$ rounds in the past.
Define
\begin{align*}
    \tilde Y_{s,t}
    = Y_{s,t} \mathds{1}\{D_s \leq m\}
    &= X_s \mathds{1} \{ D_s \leq \min (m, t-s) \}\,,
\end{align*}
which is the same as $Y_{s,t}$ except rewards that convert after more than $m$ rounds are ignored.
The learner then uses $\tilde Y_{s,t}$ to estimate a parameter that is proportional to $\theta$ using $L_2-$regularized least squares. Let $\lambda > 0$ be a regularization parameter and define
\begin{equation}
    \label{eq:cOTF}
    \thetacotf_t := \left(\sum_{s=1}^{t-1} A_s A_s^\top +\lambda I\right)^{-1} \left( \sum_{s=1}^{t-1} \tilde{Y}_{s,t} A_s \right) := V_t(\lambda)^{-1} B_t\,.
\end{equation}

We now state our main deviation result.

\begin{theorem}
\label{th:cOTF-concentration}
Let $\tau_m = \mathds{P}(D_1 \leq m)$ and $\delta \in (0,1)$. Then the following holds for all $t \leq T$ with probability at least $1 - 2\delta$.
\begin{equation}
\label{eq:w-otf-concentration}
\| \thetacotf_t - \tau_m \theta \|_{V_t(\lambda)} \leq 2f_{t,\delta} + \sum_{s=t-m}^{t-1} \norm{A_s}_{V_t^{-1}(\lambda)} \,,
\end{equation}
where
\begin{equation}
    \label{eq:explo_rate}
    f_{t,\delta} = \sqrt{\lambda} + \sqrt{2\log\left(\frac{1}{\delta} \right) + d\log \left( \frac{d\lambda + t}{d\lambda } \right)}\,.
\end{equation}
\end{theorem}

\begin{proof}
  Let $\eta_s = X_s - A_s^\top \theta$ and $\epsilon_s = \mathds{1}\{D_s \leq m\} - \tau_m$ be the two types of noise affecting our observations, both being centered Bernoulli and independent of the past conditionally on $A_s$.
  By \citep[Theorem 20.4]{lattimore2019book} it holds with probability at least $1 - 2\delta$ that for all $t \leq T$,
  \begin{align}
  \norm{\sum_{s=1}^t A_s X_s \epsilon_s}_{V_t^{-1}}^2 \leq f_{t,\delta}^2, \; \text{and,}\;
  \norm{\sum_{s=1}^t A_s \eta_s}_{V_t^{-1}}^2 \leq f_{t,\delta}^2\,.,
  \label{eq:conc1}
  \end{align}
  where we used that $X_s \in [0,1]$ for the first inequality. We comment on that step further below.

  The next step is to decompose $B_t$ with respect to the value of the censoring variables of the learner. We first rewrite $B_t$ by expliciting the value of the indicator function in each term of its sum:
  \begin{align*}
  B_t &=  \sum_{s=1}^{t-m-1}  A_s X_s \mathds{1} \{D_s \leq m\} \\
  & \quad + \sum_{s=t-m}^{t-1}  A_s X_s \mathds{1} \{D_s \leq t-s\} \\
  & =  \sum_{s=1}^{t-1}  A_s X_s \mathds{1} \{D_s \leq m\}\\
  & \quad  + \sum_{s=t-m}^{t-1}  A_s X_s (\mathds{1} \{D_s \leq t-s\} - \mathds{1}\{D_s \leq m \}),
  \end{align*}
  where we added $m$ terms in the first sum and removed them in the second one.
  The second sum now contains the terms that will eventually convert but have not been received yet.

  Now assume both events in \cref{eq:conc1} hold. Using the decomposition above, we have
  \begin{align*}
  &\| \thetacotf_t - \tau_m \theta \|_{V_t(\lambda)}\\
  & \leq \norm{V_t(\lambda)^{-1}\sum_{s=1}^{t-1}A_s X_s (\tau_m +\epsilon_s)
  - \tau_m \theta }_{V_t(\lambda)} \\
&  + \norm{ \sum_{s=t-m}^{t-1} A_s X_s }_{V_t(\lambda)^{-1}}
  \end{align*}
   The last term can be naively bounded and gives the second term of \cref{eq:w-otf-concentration}.
  We can bound the first term using \cref{eq:conc1}:
  \begin{align*}
    &\norm{V_t(\lambda)^{-1}\sum_{s=1}^{t-1}A_s X_s (\tau_m +\epsilon_s)
    - \tau_m \theta }_{V_t(\lambda)} \\
    &\leq \norm{\sum_{s=1}^{t-1} A_s X_s \epsilon_s}_{V_t(\lambda)^{-1}}   + \norm{\sum_{s=1}^{t-1} A_s \eta_s}_{V_t(\lambda)^{-1}} \\
        & \leq 2f_{t,\delta} \,,
  \end{align*}
  where both the first and second inequalities follow from the triangle inequality applied to $\norm{\cdot}_{V_t^{-1}}$ and the last from the assumption that the events in \cref{eq:conc1} hold.
\end{proof}

\begin{remark}
 The initial step of the proof in Eq~\eqref{eq:conc1} might seem loose but we explain why this term cannot easily be bounded more tightly.
 Note that the variance of $\epsilon_s$ is $\tau_m(1-\tau_m)$ and thus, when applying \citep[Theorem 20.4]{lattimore2019book}, we could obtain a tighter bound by taking it into acount, which would lead to $\tau_m^2 f_{t\delta}^2$.
But  had we included it there, it would have appeared in the expression of the upper bound, i.e. in the algorithm, and the learner would have needed its
knowledge to compute the upper bound. This was the choice made by \cite{vernade2017stochastic}. By removing
it, we pay the price of slightly larger confidence intervals (more exploration) for not having to give prior information
to the learner. We discuss other possible approaches in conclusion.
\end{remark}


\paragraph{Practical considerations}
As we mentioned already, a practical advantage of the windowing idea is that the learner need not store actions for which the feedback
has not been received indefinitely. The cut-off time is often rather long, even as much as 30 days \cite{chapelle2011empirical}.

\paragraph{Choosing the window}
The windowing parameter is often a constraint of the system and the lerner cannot choose it.
Our results show the price of this external censoring on the regret.
If the learner is able to choose $m$ the choice is somewhat delicate.
The learner effectively discards $1 - \tau_m$ proportion of the data, so ideally $\tau_m$ should be large, which corresponds to
large $m$. But there is a price for this. The learner must store $m$ actions and the confidence interval also depends (somewhat mildly) on $m$.
When the mean $\mu = \E_{D \sim \cD}[D]$ of the delay distribution $\cD$ is finite and known, then a somewhat natural choice of the windowing parameter is $m = 2\mu$.
By Markov's inequality this ensures that $\tau_m \geq 1/2$. The result continues to hold if the learner only knows an upper bound on $\mu$.

Precisely how $m$ should be chosen depends on the underlying problem. We discuss this issue in more detail in \cref{sec:analysis} where the regret analysis is provided.

\section{Algorithm}
\label{sec:algorithm}

We are now equipped to present $\OTFLinUCB$, an optimistic linear bandit algorithm that uses concentration analysis from the previous section.
The pseudocode of $\OTFLinUCB$ is given in Algorithm~\ref{alg:otflinucb}.
It accepts as input a confidence level $\delta>0$, a window parameter $m > 0$ and a regularization parameter $\lambda > 0$.
In each round the algorithm computes the estimator $\thetacotf_t$ using \cref{eq:w-otf-concentration}
and for each arm $a \in \cA_t$ computes an upper confidence bound on the expected reward defined by
\begin{equation}
  \label{eq:otf-ucb}
  U_t(a) = \shortinner{a, \thetacotf_t} + \left(2f_{t,\delta} + \sum_{s=t-m}^{t-1} \|A_s\|_{V_t(\lambda)^{-1}}\right) \norm{a}_{V_t(\lambda)^{-1}}\,.
\end{equation}
Then action $A_t$ is chosen to maximize the upper confidence bound:
\[
A_t = \argmax_a U_t(a) \,,
\]
where ties are broken arbitrarily.

\paragraph{Implementation details}
The algorithm needs to keep track of $V_t(\lambda)$ and $B_t$ as defined in \cref{eq:cOTF}.
These can be updated incrementally as actions are taken and information is received. The algorithm also uses $V_t(\lambda)^{-1}$, which can be updated incrementally
using the Sherman-–Morrison formula. In order to recompute $\alpha_{t,\delta}$ the algorithm needs to store the $m$ most recent actions, which are also
used to update $B_t$.

\paragraph{Computation complexity}
The computation complexity is dominated by three operations: (1) Updating $V_t(\lambda)$ and computing its inverse, which takes $O(d^2)$ computation steps using a rank-one update, and
(2) computing the radius of the confidence ellipsoid, which requires $O(m d^2)$ computations, one for each of the last $m$ actions. Finally, (3) iterating over the actions and computing the
upper confidence bounds, which requires $O(K_t d^2)$ computations. Hence the total computation per round is $O((K_t + m) d^2)$.

\paragraph{Space complexity}
The space complexity is dominated by: (1) Storing the matrix $V_t(\lambda)$ and (2) storing the $m$ most recent actions, which are needed to compute the least squares estimator and the upper confidence bound.
Hence, the space complexity is $O(m d + d^2)$.

\paragraph{Improved computation complexity}
Because $V_t(\lambda)$ changes in every round, the radius of the confidence set needs to be recomputed in each round, which requires $O(md^2)$ computations per round.
A minor modification reduces the computation complexity to $O(d^2)$. The idea is to notice that for any $a\in \R^d$ and $s\leq t$,
\begin{align*}
\|a\|_{V_s(\lambda)^{-1}} \geq \|a\|_{V_t(\lambda)^{-1}} \,.
\end{align*}
Hence, one can store a buffer of scalars $\{\|A_s\|_{V_s(\lambda)^{-1}},\, t - m \leq s \leq t-1\}$ at the memory cost of $O(m)$.
This slightly increases the upper confidence bounds, but not so much that the analysis is affected as we discuss in the next section.

\begin{algorithm}
\caption{\OTFLinUCB~}
\label{alg:otflinucb}
\begin{algorithmic}
\STATE {\bf Input:} Window parameter $m>0$, confidence level $\delta>0$ and $\lambda>0$.
\FOR{$t=2,\ldots,T$}
  \STATE Receive action set $\cA_t$
  \STATE Compute width of confidence interval: 
  $$\displaystyle \alpha_{t,\delta} = 2f_{t,\delta} + \sum_{s=t-m}^{t-1} \norm{A_s}_{V_t(\lambda)^{-1}}$$
  \STATE Compute the least squares estimate $\thetacotf_t$ using (\ref{eq:cOTF}) \\[0.3cm]
  \STATE Compute the optimistic action:
  $$\displaystyle A_t = \argmax_{a \in \cA_t} \shortinner{a, \thetacotf_t} + \alpha_{t,\delta} \norm{a}_{V_t(\lambda)^{-1}}$$
  \STATE Play $A_t$ and receive observations
\ENDFOR
\end{algorithmic}

\end{algorithm}

\section{Regret Analysis}
\label{sec:analysis}

Our main theorem is the following high probability upper bound on the regret of $\OTFLinUCB$.
The proof combines the ideas from \cite{abbasi2011improved} with a novel argument to handle the confidence bound in \cref{th:cOTF-concentration}, which has a more
complicated form than what usually appears in the analysis of stochastic linear bandits.

\begin{theorem}\label{thm:upper}
With probability at least $1 - 2 \delta$ the regret of $\OTFLinUCB$ satisfies
\begin{multline*}
R(T, \theta)
\leq \frac{4 f_{n,\delta}}{\tau_m} \sqrt{2dT \log\left(\frac{d \lambda + T}{d\lambda}\right)} \\
+ \frac{4md}{\tau_m} \log\left(\frac{d\lambda + T}{d\lambda}\right)\,.
\end{multline*}
\end{theorem}

\begin{proof}
Let $\alpha_{t,\delta} = 2 f_{t,\delta} + \sum_{s=t-m}^{t-1} \norm{A_s}_{V_t(\lambda)^{-1}}$, which is chosen so that the upper confidence bound for action $a$ in round $t$ is
\begin{align*}
U_t(a) = \shortinner{a, \thetacotf_t} + \alpha_{t,\delta} \norm{a}_{V_t(\lambda)^{-1}}\,.
\end{align*}
and $A_t = \argmax_{a \in \cA_t} U_t(a)$. By \cref{th:cOTF-concentration}, with probability at least $1 - 2\delta$
it holds that $\snorm{\thetacotf_t - \tau_m \theta}_{V_t(\lambda)} \leq \alpha_{t,\delta}$ for all $t$. Assume for the remainder that the above event holds. Then
\begin{align*}
\inner{A_t^*, \theta}
&= \frac{1}{\tau_m} \shortinner{\tau_m \theta, A_t^*} \\
&\leq \frac{1}{\tau_m} \left(\shortinner{\thetacotf_t, A_t^*} + \alpha_{t,\delta} \norm{A_t^*}_{V_t(\lambda)^{-1}}\right)
= \frac{U_t(A_t^*)}{\tau_m} \\
&\leq \frac{U_t(A_t)}{\tau_m}
= \frac{1}{\tau_m} \left(\shortinner{\thetacotf_t, A_t} + \alpha_{t,\delta} \norm{A_t}_{V_t(\lambda)^{-1}}\right)\,.
\end{align*}
Therefore the regret in round $t$ is bounded by
\begin{align*}
\inner{A_t^* - A_t, \theta}
&\leq \frac{1}{\tau_m} \shortinner{A_t, \thetacotf_t - \tau_m \theta} + \frac{\alpha_{t,\delta}}{\tau_m} \norm{A_t}_{V_t(\lambda)^{-1}} \\
& \leq \frac{2\alpha_{t,\delta}}{\tau_m} \norm{A_t}_{V_t(\lambda)^{-1}}\,,
\end{align*}
where the second inequality follows from Cauchy-Schwarz.
We now substitute the value of $\alpha_{t,\delta}$ and bound the overall regret by
\begin{multline}
R(T, \theta)
\leq \frac{4}{\tau_m} \sum_{t=1}^T f_{t,\delta} \norm{A_t}_{V_t^{-1}} \\
+ \frac{2}{\tau_m} \sum_{t=1}^T \norm{A_t}_{V_t(\lambda)^{-1}} \sum_{s=t-m}^{t-1} \norm{A_s}_{V_t(\lambda)^{-1}}\,. \label{eq:regret-decomp}
\end{multline}
The first sum is bounded in the same way as the standard setting \cite{abbasi2011improved}:
\begin{align*}
\frac{4}{\tau_m} \sum_{t=1}^T f_{t,\delta} \norm{A_t}_{V_t(\lambda)^{-1}}
&\leq \frac{4f_{n,\delta}}{\tau_m} \sqrt{T \sum_{t=1}^T \norm{A_t}_{V_t(\lambda)^{-1}}^2} \\
&\leq \frac{4f_{n,\delta}}{\tau_m}  \sqrt{2dT \log\left(\frac{d \lambda + T}{d\lambda}\right)}\,,
\end{align*}
where the first inequality follows from Cauchy-Schwarz and the second
from the elliptical potential lemma \citep[Lemma 19.4]{lattimore2019book}.
For the second sum in \cref{eq:regret-decomp} we introduce a new trick. Using the fact that $ab \leq (a^2+b^2)/2$,
\begin{align*}
&\frac{2}{\tau_m} \sum_{t=1}^T \norm{A_t}_{V_t(\lambda)^{-1}} \sum_{s=t-m}^{t-1} \norm{A_s}_{V_t(\lambda)^{-1}} \\
&\leq \frac{1}{\tau_m} \sum_{t=1}^T \sum_{s=t-m}^{t-1}\left(\norm{A_t}_{V_t(\lambda)^{-1}}^2 + \norm{A_s}_{V_t(\lambda)^{-1}}^2\right)  \\
&\leq \frac{1}{\tau_m} \sum_{t=1}^T \sum_{s=t-m}^{t-1}\left(\norm{A_t}_{V_t(\lambda)^{-1}}^2 + \norm{A_s}_{V_s(\lambda)^{-1}}^2\right)  \\
&\leq \frac{2m}{\tau_m} \sum_{t=1}^T \norm{A_t}_{V_t(\lambda)^{-1}}^2
\leq \frac{4md}{\tau_m} \log\left(\frac{d\lambda + T}{d\lambda}\right)\,,
\end{align*}
where in the second inequality we used the fact that for $s \leq t$, $V_s(\lambda) \leq V_t(\lambda)$\footnote{here we denote $A\leq B$ if for any $x\in \R_+^d$, $\|x\|_A \leq \|x\|_B$ }, so that $V_s(\lambda)^{-1} \geq V_t(\lambda)^{-1}$.
Substituting the previous two displays into \cref{eq:regret-decomp} completes the proof.
\end{proof}

\begin{remark}
The choice of $m$ is left to the learner. It influences the bound in two ways: (1) The lower-order term
is linear in $m$, which prevents the user from choosing $m$ very large.
On the other hand, $\tau_m$ is increasing in $m$, which pushes the user in the opposite direction.
Designing an adaptive algorithm that optimizes the choice of $m$ online remains a challenge for the future.
\end{remark}

\paragraph{Lower bound}
We now provide a non-asymptotic minimax lower bound for $K$-armed stochastic Bernoulli bandits showing that in the windowed
setting there is an unavoidable dependence on $\tau_m$.
Note, an asymptotic problem-dependent bound for this setting was already known \cite{vernade2017stochastic}. Although our results are specialized to the finite-armed bandit model, we expect that standard analysis for other action sets should follow along the same lines as \citep{lattimore2019book} (\S24).

\begin{theorem}\label{thm:lower}
For any policy $\pi$ and $K > 1$ and $T \geq 1$ and $\tau_m \in (0,1)$ there exists a $K$-armed Bernoulli bandit such that
$R(T, \theta) \geq c \min\{T, \sqrt{T K / \tau_m}\}$, where $c >  0$ is a universal constant.
\end{theorem}

Interestingly, the dependence on $\tau_m$ appears in the square root, while in our upper bounds it is not.
We speculate that the upper bound is loose. In fact, were $\tau_m$ known it would be possible to improve our upper bounds by using
confidence intervals based on Bernstein's inequality that exploit the reduced variance that is a consequence of $\tau_m$ being small.
When $\tau_m$ is unknown you might imagine estimating the variance. We anticipate this should be possible, but the complexity of the algorithm and analysis would greatly increase.

\section{Thompson sampling}
\label{sec:ts}
The standard implementation of Thompson sampling for linear bandits without delays and Gaussian noise is to sample
$\tilde \theta_t \sim \cN(\hat \theta_t, V_t(\lambda)^{-1})$
where $\hat \theta_t$ is the usual regularized least squares estimator
\begin{align*}
    \hat \theta_t = V_t(\lambda)^{-1} \sum_{s=1}^{t-1} A_s X_s\,.
\end{align*}
The algorithm then chooses
\begin{align*}
    A_t = \argmax_{a \in \cA_t} \shortinner{a, \tilde \theta_t}\,.
\end{align*}
This algorithm corresponds to Thompson sampling (or posterior sampling) when the prior is Gaussian with zero mean and $\lambda I$ covariance. No frequentist analysis exists for this algorithm, but empirically it performs very well. In the delayed setting the $(X_s)_{s=1}^{t-1}$ are not available to the learner at time $t$. Nevertheless, it is possible to propose a randomized algorithm in the spirit of Thompson sampling.
To motivate our choices, recall that the standard concentration analysis for least squares regression by \cite{abbasi2011improved} shows that with high probability
\begin{align*}
    \snorm{\hat \theta_t - \theta}_{V_t(\lambda)} \leq f_{t,\delta}\,.
\end{align*}
In the delayed setting, on the other hand, \cref{th:cOTF-concentration} shows that
\begin{align}
    \snorm{\thetacotf_t - \theta}_{V_t(\lambda)} \leq f_{t,\delta} + \sum_{s=t-m}^{t-1} \norm{A_s}_{V_s(\lambda)^{-1}} \\
    = f_{t,\delta}\left(1 + \frac{\sum_{s=t-m}^{t-1} \norm{A_s}_{V_s(\lambda)^{-1}}}{f_{t,\delta}}\right)\,.
    \label{eq:ts-delay-conf}
\end{align}
An ansatz guess for a sampling algorithm that uses the delayed least squares estimator is to compute $\thetacotf_t$ and then sample 
\begin{align*}
    \tilde \theta_t \sim \cN\left(\thetacotf_t, \beta_{t,\delta} V_t(\lambda)^{-1}\right)\,,
\end{align*}
where $\beta_{t,\delta} = 1 + (\sum_{s=t-m}^{t-1} \norm{A_s}_{V_t(\lambda)^{-1}})/f_{t,\delta}$.
The choice of $\beta_{t,\delta}$ is rather heuristic. A more conservative choice would be the right-hand side of \cref{eq:ts-delay-conf}. The resulting algorithm roughly corresponds to sampling
from the confidence set used by our optimistic algorithm. Although this sacrifices certain empirical advantages, we expect the analysis techniques by \cite{agrawal2013thompson,abeille2017linear}
could be applied to prove a frequentist regret bound for this algorithm.

\begin{remark}
Algorithms based on adding noise to an empirical estimate are often referred to as `follow the perturbed leader', which has been effectively applied in a variety of settings \cite{abeille2017linear,KSW18}.
An advantage of sampling approaches is that the optimization problem to find $A_t$ is a linear program, which for large structured action sets may be more efficient than finding the arm maximizing an upper confidence bound.
\end{remark}

\begin{remark}
A genuine implementation of Thompson sampling would require a prior on the space of delay distributions as well as the unknown parameter.
We are not hopeful about the existence of a reasonable prior for which computing or sampling from the posterior is efficient.
\end{remark}

\begin{algorithm}
\caption{\OTFLinTS:}
\label{alg:otflints}
\begin{algorithmic}[1]
\STATE {\bf Input:} Window parameter $m > 0$, confidence level $\delta>0$, $\lambda>0$.
\FOR{$t=2,\ldots,T$}
  \STATE Receive action set $\cA_t$
  \STATE Compute width of confidence interval:
  $$\displaystyle \beta_{t,\delta} = 1 + \frac{\sum_{s=t-m}^{t-1} \norm{A_s}_{V_s(\lambda)^{-1}}}{f_{t,\delta}}$$
  \STATE Compute the least squares estimate $\thetacotf_t$ using (\ref{eq:cOTF}) \\[0.3cm]
  \STATE Sample $\tilde \theta_t \sim \cN(\thetacotf_t, \beta_{t,\delta} V_t(\lambda)^{-1})$ \\[0.3cm]
  \STATE Compute action $A_t = \argmax_{a \in \cA_t} \shortinner{a, \tilde \theta_t}$ \\[0.3cm]
  \STATE Play $A_t$ and receive observations
\ENDFOR
\end{algorithmic}
\end{algorithm}

\section{Related Work}
\label{sec:related}

Delays in the environment response is a frequent phenomenon that may take many different forms and should be properly modelled to design appropriate decision strategies. For instance, in an early work on applications of bandit algorithms to clinical trials, \cite{eick1988two} uses `delays' to model the survival time of the patients, in which case delays are the reward rather than external noise, which is a radically different problem to ours. In the same vein, in the sequential stochastic shortest path problem \cite{talebi2017stochastic}, the learner aims at minimising the routing time in a network.

Another example is parallel experimentation, where delays force the learner to make decisions under temporary uneven information. For instance,
\cite{desautels2014parallelizing,grover2018best} consider the problem of running parallel experiments that do not all end simultaneously. They propose a Bayesian way of handling uncertain outcomes to make decisions: they sample \emph{hallucinated} results according to the current posterior. The related problem of gradient-based optimization with delayed stochastic gradient information is studied by \cite{agarwal2011distributed}.

In online advertising, delays are due to the natural latency in users' responses. However, in many works on bandit algorithm, delays are ignored as a first approximation.
In the famous empirical study of Thompson sampling \cite{chapelle2011empirical}, a section is dedicated to analyzing the impact of delays on either Thompson sampling or $\LinUCB$.
While this is an early interest for this problem, they only consider fixed, non-random, delays of 10, 30 or 60 minutes.
Similarly, in \cite{mandel2015towards}, the authors conclude that randomized policies are more robust to this type of latencies.
The general problem of online learning under \emph{known} delayed feedback is addressed in \cite{joulani2013online}, including full information settings and partial monitoring, and we refer the interested reader to their references on those topics. The most recent and closest work to ours is \cite{zhou2019learning}.
The main difference with our approach is that they make strong assumptions on the distribution of the delays, while not having any censoring of the feedback.
In that sense their problem is easier than ours because delays are fully observed. Nonetheless, the key idea of their algorithm is reminiscent to ours: they inscrease the exploration bonus by a quantity that corresponds to the amount of missing data at each round, which is observable in their case, not in ours.
Recent work \cite{li2019bandit} address the case of unknown delays.
The idea of ambiguous feedback, where delays are partially unknown is introduced in \cite{vernade2017stochastic}.

Many models of delays for online advertising have been proposed to estimate conversion rates in an offline fashion: e.g. \cite{yoshikawa2018nonparametric} (non-parametric) or \cite{chapelle2014modeling, DiemertMeynet2017} (generalized linear parametric model).

An alternative, harder model relies only on anonymous feedback \cite{cesa2018nonstochastic,pike2017bandits,arya2019randomized}: the rewards, when observed, cannot be directly linked to the action that triggered them in the past, and the learner has to deal with mixing processes. Finally, a recent full-information setting \cite{mann2018learning} suggests to incorporate intermediate feedback correlated with the rewards.

\section{Experiments}
\label{sec:experiments}

\begin{figure*}[]
\centering
\includegraphics[width=0.32\textwidth]{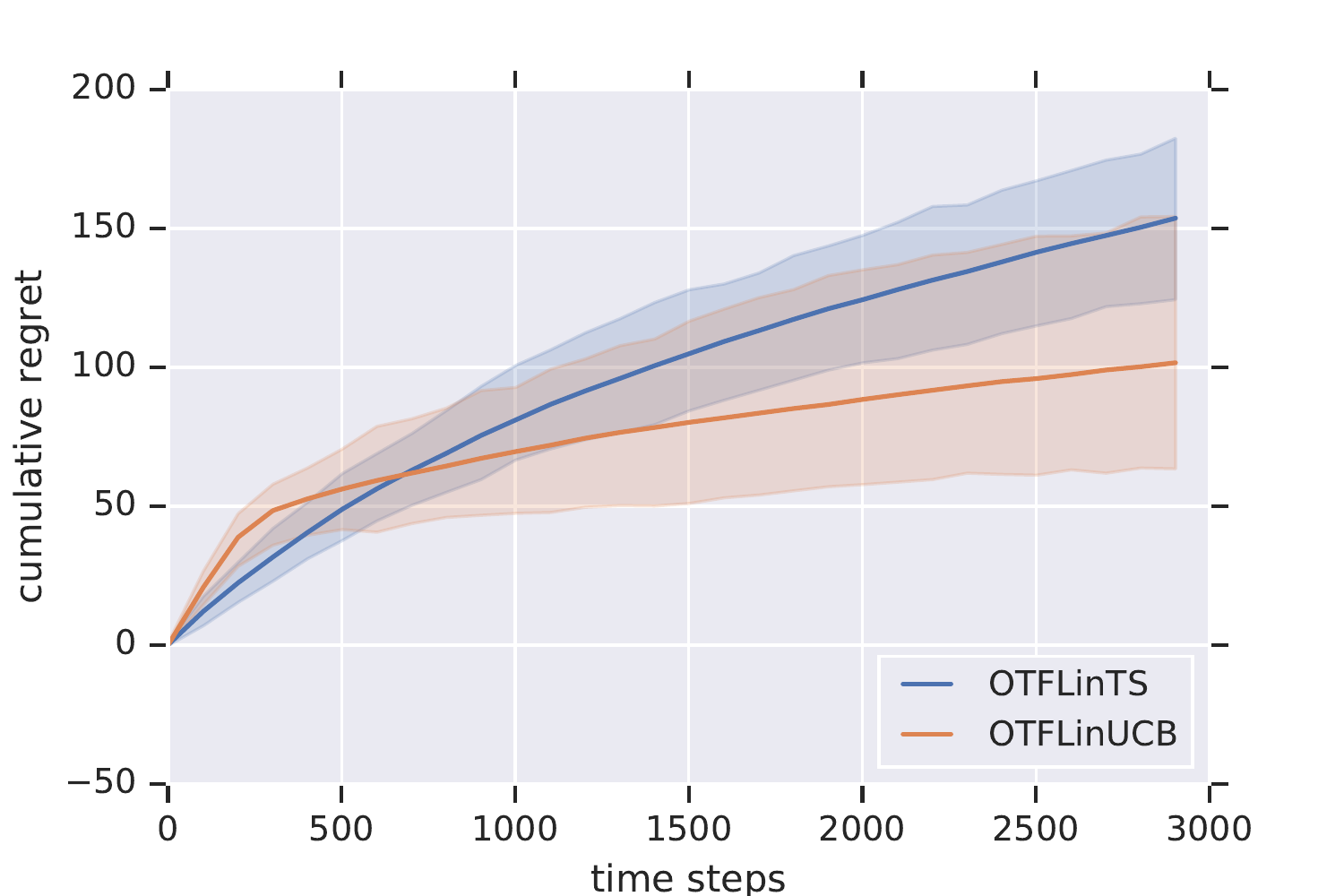}
\includegraphics[width=0.32\textwidth]{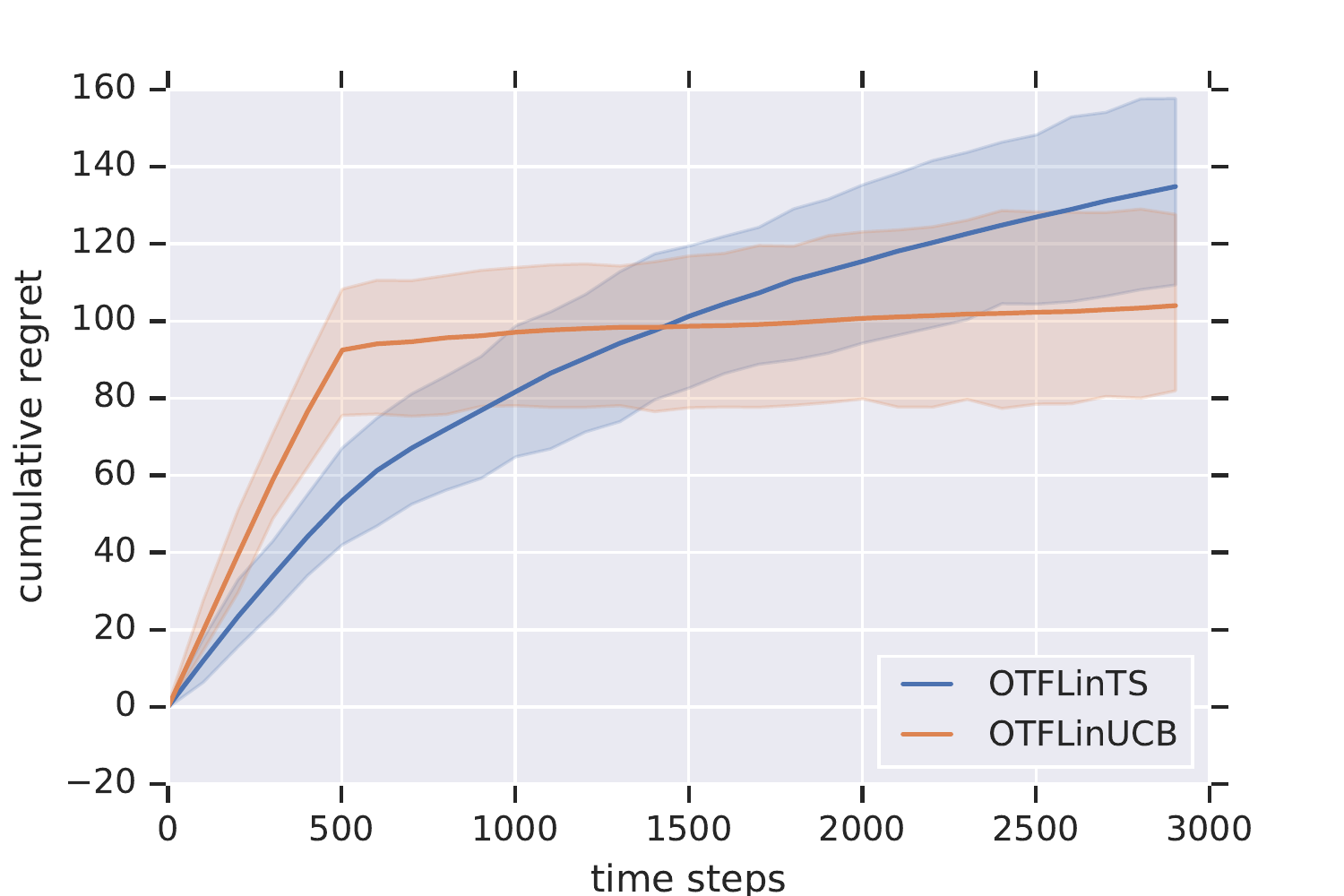}
\includegraphics[width=0.32\textwidth]{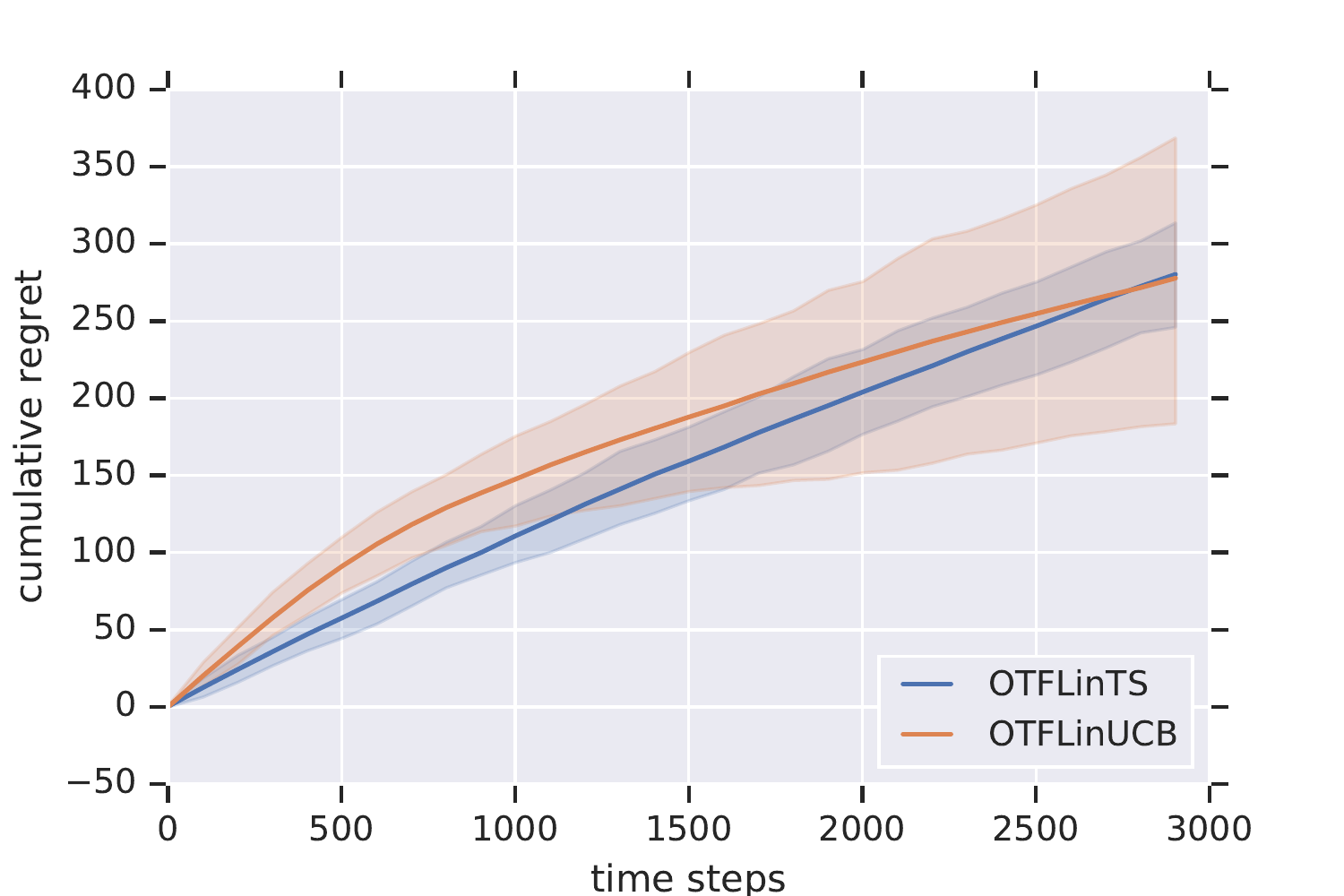}
\caption{Results of our simulations. From left to right, the plots
report the performance of the two algorithms with $(m,\mu)=\{(100,100), (500, 100), (100,500)\}$.Results are averaged over 100 independent runs.  \label{fig:exp-results} }
\end{figure*}

In this section we illustrate the two realistic settings handled by this work.
The first case below, with geometrically distributed delays,
corresponds to the empirical study done by \cite{chapelle2014modeling},
already reproduced in simulations by \cite{vernade2017stochastic}.
The second case, with arbitrary heavy-tailed delays, corresponds to
another use-case we extracted from the data released by \cite{DiemertMeynet2017}.\footnote{The code for all data analysis and simulations is available at https://sites.google.com/view/bandits-delayed-feedback}

\begin{remark}
  To the best of our knowledge, there is no competitor for this problem.
  In particular, despite the similarity of the algorithm DUCB of \cite{zhou2019learning} with ours,
  it cannot be implemented when delays are not observed. Specifically, DUCB, maintains a quantity $G_t$
  which is equal to the exact amount of missing data (delayed feedback not converted yet).
  In our case, this quantity is not observable. The same comment applies for the QPM-D algorithm of \cite{joulani2013online}
  and similar queue-based approaches. On the other end, existing algorithms for unknown delays are not derived for linear bandits with arbitrary action sets.
\end{remark}

\begin{figure}[hbt]
  \begin{center}
  \includegraphics[width=0.35\textwidth]{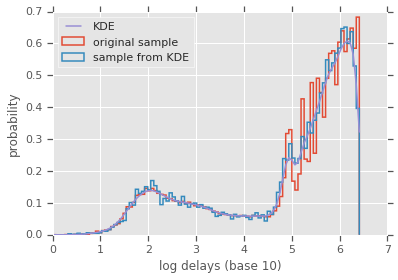}
  \caption{Empirical distribution of the $\log_{10}$ delays.  \label{fig:criteo-experiments} }
  \vspace{-0.5cm}
  \end{center}
\end{figure}

\paragraph{Well-behaved delays} In the datasets analyzed\footnote{The datasets were not released.} by \cite{chapelle2014modeling}, delays are empirically shown to have an exponential decay. As in \cite{vernade2017stochastic}, we run realistic simulations based on the orders of magnitude provided in their study.
We arbitrarily choose $d=5$, $K=10$.
We fix the horizon to $T=3000$, and we choose a geometric delay distribution with mean $\mu=\mathds{E}[D_t] \in \{100, 500\}$.
In a real setting, this would correspond to an experiment that lasts 3h, with average delays of 6 and  30  minutes\footnote{Note that in \cite{chapelle2014modeling}, time is rather measured in hours and days.} respectively.
The online interaction with the environment is simulated: we fix $\theta = \{1/\sqrt{d}, \ldots, 1/\sqrt{d}\}$ and at each round we sample and normalize $K$ actions from $\{0,1\}^d$. All result are averaged over 50 independent runs.
We show three cases on Figure~\ref{fig:exp-results}: $(m,\mu)=(100,100)$ ($\tau_m=0.63$) when delays are short and even though the timeout is set a bit too low, the algorithm manages to learn within a reasonable time range; $(m,\mu)=(500,100)$ ($\tau_m=0.993$) when the window parameter is large enough so almost all feedback is received; and $(m,\mu)=(100,500)$ when the window is set too low and the regret suffers from a high $1/\tau_m = 5.5 $.
Note that the notion of `well-tuned` window is dependent on the memory constraints one has. In general, the window is a system parameter that the learner cannot easily set to their convenience.
In all three cases, $\OTFLinUCB$ performs better than $\OTFLinTS$, which we believe is due to the rough posterior approximation.
Comparing the two leftmost figures, we see the impact of increasing the window size: the log regime starts earlier but the asymptotic regret is similar -- roughly $R(T)=100$ in both cases. The rightmost figure, compared with the leftmost one, shows the impact of a high level of censoring. Indeed, with  $1/\tau_m = 5.5 $, only few feedback are received within the window and the algorithm will need much longer to learn. This causes the regret increase, and also explains why at $T=3000$, both algorithms are still in the linear regime.

\begin{figure}[hbt]

  \begin{center}
  \includegraphics[width=0.40\textwidth]{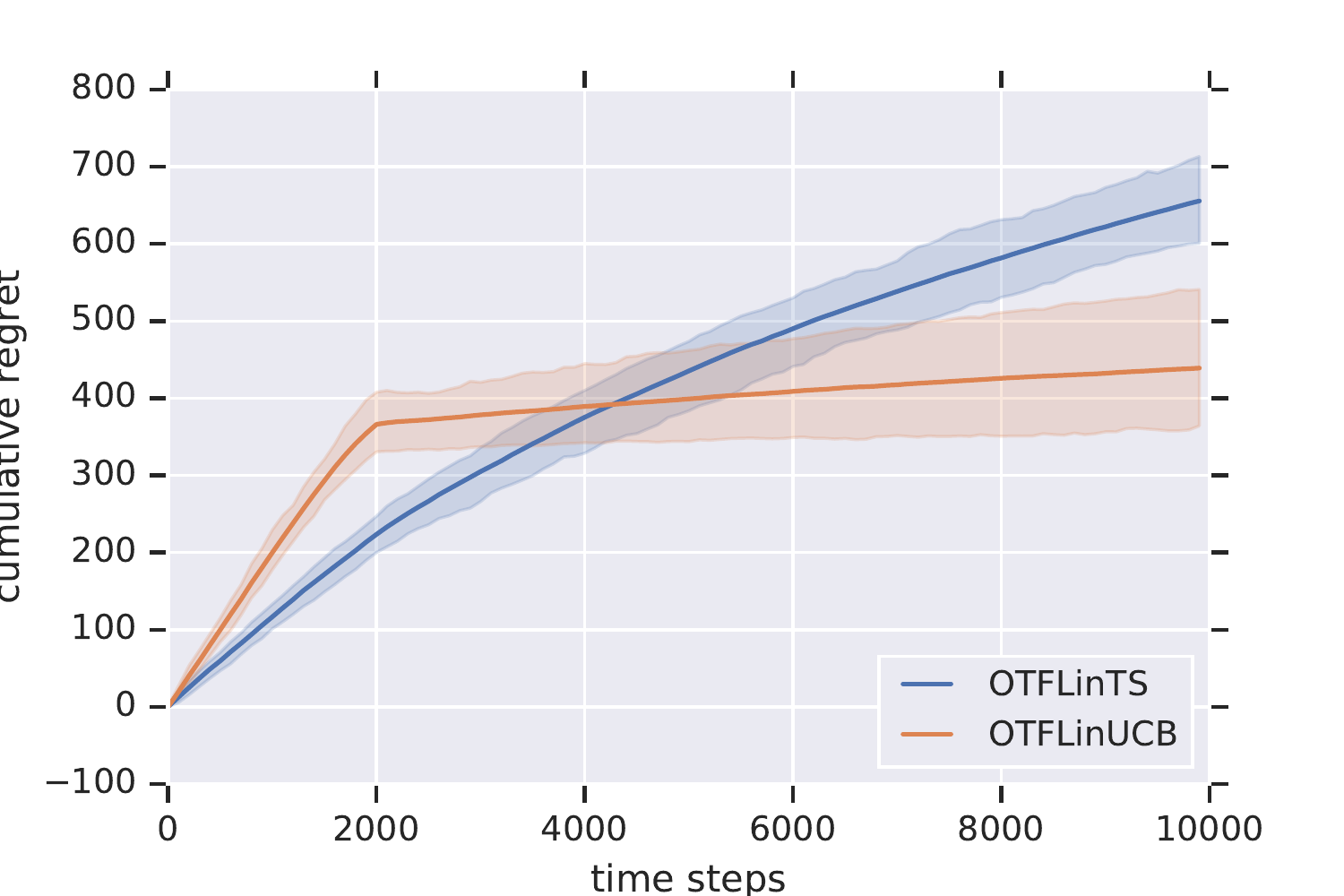}
  \caption{Simulation with realistic delays and $m=2000$, sampled according to the distribution of Figure~\ref{fig:criteo-experiments}. \label{fig:criteo-experiments-results} }
  \vspace{-0.5cm}
  \end{center}

\end{figure}

\paragraph{Heavy-tailed delays.}

Interestingly, the more recent dataset released by \cite{DiemertMeynet2017}\footnote{https://ailab.criteo.com/criteo-attribution-modeling-bidding-dataset/}
features heavy-tailed delays, despite being sourced from a similar online marketing problem in the same company. We run simulations in the same setting as before,
but this time we generate the delays according to the empirical delay distribution extracted from the dataset. On  Figure~\ref{fig:criteo-experiments}, we show the distribution of the log-delays fitted with a gaussian kernel using the Scipy library.
A high proportion of those delays are between $10^5$ and $10^6$ so we rescaled them by a factor $0.01$ to maintain the experiment horizon to the
reasonable value of $T=10^4$. Nevertheless, many rewards will not be observed within the horizon. We compare $\OTFLinUCB$ and $\OTFLinTS$
with window parameter $m=2000$. The results are shown on Figure~\ref{fig:criteo-experiments-results} and show that $\OTFLinUCB$ is able to cope with heavy-tailed delays reasonably well. It is interesting to observe that the behavior of both algorithms is similar to the central plot of Figure~\ref{fig:exp-results}, when the window is larger than the expectation of the delays.

\section{Discussion}
\label{sec:discussion}

We introduced the delayed stochastic linear bandit setting and proposed two algorithms. The first uses the optimism principle in combination with ellipsoidal
confidence intervals while the second is inspired by Thompson sampling and follow the perturbed leader.

There are a number of directions for future research, some of which we now describe.

\paragraph{Improved algorithms}
Our lower bound suggests the dependence on $\tau_m$ in \cref{thm:upper} can be improved. Were this value known we believe that using a concentration analysis that
makes use of the variance should improved the dependence to match the lower bound.
When $\tau_m$ is not known, however, the problem becomes more delicate. One can envisage various estimation schemes, but the resulting algorithm and analysis are likely
to be rather complex.

\paragraph{Generalized linear models}
When the rewards are Bernoulli it is natural to replace the linear model with a generalized linear model.
As we remarked already, this should be possible using the machinery of \cite{filippi2010parametric}.

\paragraph{Thompson sampling}
Another obvious question is whether our variant of Thompson sampling/follow the perturbed leader admits a regret analysis.
In principle we expect the analysis by \cite{agrawal2013thompson} in combination with our new ideas in the proof of \cref{thm:upper} can be combined
to yield a guarantee, possibly with a different tuning of the confidence width $\beta_{t,\delta}$.
Investigating a more pure Bayesian algorithm that updates beliefs about the delay distribution as well as unknown parameter is also a fascinating open question, though
possibly rather challenging.

\paragraph{Refined lower bounds}
Our current lower bound is proven when $\cA_t = \{e_1,\ldots,e_d\}$ is the standard basis vectors for all $t$.
It would be valuable to reproduce the results where $\cA_t$ is the unit sphere or hypercube, which should be a straightforward adaptation of the results in \citep[\S24]{lattimore2019book}.

\section*{Acknowledgements}
This work was started when CV was at Amazon Berlin and at OvGU Magdeburg, working closely with AC, GZ, BE and MB.
 The work of A. Carpentier is partially supported by the Deutsche Forschungsgemeinschaft (DFG) Emmy Noether grant MuSyAD (CA 1488/1-1), by the DFG - 314838170, GRK 2297 MathCoRe, by the DFG GRK 2433 DAEDALUS (384950143/GRK2433), by the DFG CRC 1294 'Data Assimilation', Project A03,  by the UFA-DFH through the French-German Doktorandenkolleg CDFA 01-18 and by the UFA-DFH through the French-German Doktorandenkolleg CDFA 01-18 and by the SFI Sachsen-Anhalt for the project RE-BCI.
Major changes and improvements were made thanks to TL at DeepMind later on.
CV wants to thank Csaba Szepesv\'ari for his useful comments and discussions, and Vincent Fortuin for precisely reading and commenting.

\bibliographystyle{icml2020}
\bibliography{refsbandits}

\newpage
\onecolumn
\ifsup

\appendix

\section{Proof of Theorem~\ref{thm:lower}}

For $p, q \in (0,1)$ let $d(p,q) = p \log(p/q) + (1 - p) \log((1-p)/(1-q))$ be the relative entropy between Bernoulli distributions with biases $p$ and $q$ respectively. For $\theta \in [0,1]^K$ let $\E_\theta$ denote the expectation when the algorithm interacts with the Bernoulli bandit determined by $\theta \in [0,1]^K$.
Let $\theta = (1/2 + \Delta, 1/2,\ldots,1/2)$ where $\Delta \in (0,1/4)$ is some parameter to be tuned subsequently.
Then let
\begin{align*}
    i = \argmin_{k > 1} \E_{\theta}[N_k(T)]\,.
\end{align*}
By the pigeonhole principle it follows that $\E_{\theta}[N_i(T)] \leq T/(K-1)$. Then define
$\phi \in [0,1]^K$ so that $\phi_j = \theta_j$  for all $j \neq i$ and $\phi_i = 1/2 + 2\Delta$.
By the definitions of $\theta$ and $\phi$ we have
\begin{align*}
R_\theta(T) \geq \Delta (T - \E_\theta[N_1(T)]) \quad \text{and} \quad
R_\phi(T) \geq \Delta \E_\phi[N_1(T)]\,,
\end{align*}
which means that
\begin{align*}
R_\theta(T) \geq \frac{T \Delta}{2} \PP_\theta(N_1(T) \leq T/2) \quad \text{and} \quad
R_\phi(T) \geq \frac{T \Delta}{2} \PP_\phi(N_1(T) > T/2) \,.
\end{align*}
Summing the two regrets and applying the Bretagnolle-Huber inequality shows that
\begin{align*}
    R_\theta(T) + R_\phi(T)
    &\geq \frac{T \Delta}{2} \left(\PP_\theta(N_1(T) \leq T/2) + \PP_\phi(N_1(T) > T/2)\right) \\
    &\geq \frac{T \Delta}{4} \exp\left(-KL(\PP_\theta, \PP_\phi)\right)\,.
\end{align*}
The next step is to calculate the relative entropy between $\PP_\theta$ and $\PP_\phi$.
Both bandits behave identically on all arms except action $i$. When action $i$ is played the learner effectively observes a reward with bias either $\tau_m / 2$ or $\tau_m(1/2+2 \Delta)$. Therefore
\begin{align*}
KL(\PP_\theta, \PP_\phi) = \E_\theta\left[N_i(T)\right] d(\tau_m / 2, \tau_m(1/2+2\Delta))\,.
\end{align*}
Upper bounding the relative entropy by the $\chi$-squared distance shows that
\begin{align*}
d(\tau_m/2, \tau_m (1/2 + 2\Delta))
\leq \frac{2\left(\tau_m / 2 - \tau_m(1/2 + 2\Delta)\right)^2}{\tau_m(1/2 - 2\Delta)}
\leq 32 \tau_m \Delta^2\,,
\end{align*}
where we used the assumption that $2\Delta \leq 1/4$.
Therefore
\begin{align*}
    KL(\PP_\theta, \PP_\phi) \leq 32 \tau_m \Delta^2 \E_\theta[N_i(T)]
    \leq \frac{32 \tau_m \Delta^2 T}{K-1}\,.
\end{align*}
Finally we conclude that
\begin{align*}
    R_\theta(T) + R_\phi(T) \geq \frac{T \Delta}{4} \exp\left(-\frac{32 \tau_m \Delta^2 T}{K-1}\right)\,.
\end{align*}
The result follows by tuning $\Delta$.

\fi

\end{document}